\documentclass{article}

\usepackage{microtype}
\usepackage{graphicx}
\usepackage{subfigure}
\usepackage{booktabs} 

\usepackage{hyperref}

\usepackage[accepted]{icml2023}

\usepackage{amsmath}
\usepackage{amssymb}
\usepackage{mathtools}
\usepackage{amsthm}

\usepackage[capitalize,noabbrev]{cleveref}

\theoremstyle{plain}
\newtheorem{theorem}{Theorem}[section]

\theoremstyle{definition}
\newtheorem{definition}[theorem]{Definition}

\theoremstyle{remark}

\usepackage[textsize=tiny]{todonotes}

\usepackage[utf8]{inputenc} 
\usepackage[T1]{fontenc}    
\usepackage{hyperref}       
\usepackage{url}            
\usepackage{booktabs}       
\usepackage{amsfonts}       
\usepackage{nicefrac}       
\usepackage{microtype}      
\usepackage{xcolor}         
\usepackage{subfiles}

\usepackage{enumitem}
\usepackage{amsmath}
\usepackage{nicefrac, xfrac}
\newcommand\normx[1]{\Vert#1\Vert}
\usepackage{lipsum}
\usepackage{amsmath, amsthm}
\usepackage{mathtools}
\newcommand\myeq{\stackrel{\mathclap{\normalfont\mbox{def}}}{=}}
\newcommand{\eg}{\textit{e}.\textit{g}. }

\icmltitlerunning{Gradient Surgery for One-shot Unlearning on Generative Model}

\begin{document}

\twocolumn[
\icmltitle{Gradient Surgery for One-shot Unlearning on Generative Model}



\icmlsetsymbol{equal}{*}

\begin{icmlauthorlist}
\icmlauthor{Seohui Bae}{comp}
\icmlauthor{Seoyoon Kim}{comp}
\icmlauthor{Hyemin Jung}{comp}
\icmlauthor{Woohyung Lim}{comp}
\end{icmlauthorlist}

\icmlaffiliation{comp}{LG AI Research, Seoul, South Korea}

\icmlcorrespondingauthor{Seohui Bae}{seohui.bae@lgresearch.ai}
\icmlcorrespondingauthor{Woohyung Lim}{w.lim@lgresearch.ai}

\icmlkeywords{deep unlearning, generative model, privacy}

\vskip 0.3in
]



\printAffiliationsAndNotice{} 

\begin{abstract}
Recent regulation on right-to-be-forgotten emerges tons of interest in unlearning pre-trained machine learning models. While approximating a straightforward yet expensive approach of retrain-from-scratch, recent machine unlearning methods unlearn a sample by updating weights to remove its influence on the weight parameters. In this paper, we introduce a simple yet effective approach to remove a data influence on the deep generative model. Inspired by works in multi-task learning, we propose to manipulate gradients to regularize the interplay of influence among samples by projecting gradients onto the normal plane of the gradients to be retained. Our work is agnostic to statistics of the removal samples, outperforming existing baselines while providing theoretical analysis for the first time in unlearning a generative model. 
\end{abstract}

\vspace{-0.3in}
\section{Introduction}
Suppose a user wants to get rid of his/her face image anywhere in your facial image generation application - including the database and the generative model on which it is trained. Is the expensive retrain-from-scratch the only solution for this kind of request? As the use of personal data has been increased in training the machine learning models for online service, meeting individual demand for privacy or the rapid change in the legislation of General Data Protection Registration (GDPR) is inevitable to ML service providers nowadays. This request on `Right-To-Be-Forgotten (RTBF)' might be a one-time or in-series, scaling from a feature to a number of tasks, querying single instance to multiples. 
A straightforward solution for unlearning a single data might be to retrain a generative model from scratch without data of interest. This approach, however, is intractable in practice considering the grand size and complexity of the latest generative models ~\citep{rombach2022high, child2020very} and the continual request for removal. 

Unlearning, thereafter, aims to approximate this straightforward-yet-expensive solution of retrain-from-scratch time and computation efficiently. First-order data-influence-based approximate unlearning is currently considered the state-of-the-art approach to unlearning machine learning models in general. Grounded by the notion of data influence~\citep{koh2017understanding}, a simple one-step Newton's update certifies sufficiently small bound between retrain-from-scratch~\citep{guo2020certified}. Nonetheless, those relaxations are infeasible to the non-convex deep neural networks (\eg generative model) where the gap is not certifiably bounded and the process of computing the inverse of hessian is intractable. Several recent works also have affirmed that these relaxed alternatives perform poorly on deep neural networks~\citep{golatkar2021mixed,liu2022continual} and even that on generative models have not been explored yet.

\paragraph{Contribution} In this work, we propose a novel one-shot unlearning method for unlearning samples from pre-trained deep generative model. Relaxing the definition of influence function on parameters in machine unlearning ~\citep{koh2017understanding,basu2020second}, we focus on the influence of a single data on the \textit{test loss} of the others and propose a simple and cost-effective method to minimize this inter-dependent influence to approximate retrain-from-scratch. We summarize our contributions as follows:
\vspace{-0.1in}
\begin{itemize} 
    \item We propose to annul the influence of samples on generations with simple gradient manipulation. 
    \vspace{-0.05 in}
    \item Agnostic to removal statistics and thus applied to any removals such as a single data, a class, some data feature, etc.  
    \vspace{-0.05 in}
    \item Grounded by a theoretical analysis bridging standard machine unlearning to generative model. 
\end{itemize}

\section{Gradient Surgery for One-shot Data Removals on Generative Model} 
\label{sec:method}

    \paragraph{Notations} {Let $D=\{\textbf{x}_i\}_{i=1}^{N} \subseteq \mathcal{X}$ be the training data where $\textbf{x}_i \in \mathcal{X}$ is input. Let $D_f \subseteq D$ be a subset of training data that is to be forgotten (\textit{i.e.} forget set) and $D_r = D \setminus D_f$ be remaining training data of which information we want to retain. Recall that the goal of unlearning is to approximate the deep generative model retrained from scratch with only $D_r$, which we denote as $f_{\theta^*}$ parameterized by $\theta^*$. Then, our goal is to unlearn $D_f \subseteq D$ from a converged pre-trained generator $f_{\hat{\theta}}$ by updating the parameter $\hat{\theta} \rightarrow{\theta^{-}}$, where $\theta^{-}$ represents the updated parameters obtained after unlearning. 
    }
\vspace{-0.2in}
\paragraph{Proposed method}{ 
Given a generative model that models the distribution of training data $p(D)$, a successful unlearned model that unlearns $D_f$ would be what approximates $p(D_r)$, the distribution of $D_r$, as if it had never seen $D_f$. The only case where the unlearned model generates samples similar to $x\in D_f$ is when $p(D_f)$ and $p(D_r)$ happen to be very close from the beginning. Under this goal, a straight-forward objective given the pre-trained model approximating $p(D)$ is to make the output of generation to \textit{deviate from} $p(D_f)$, which could be simply formulated as the following:  
\begin{equation} \label{eq:eq1}
\begin{aligned}
\max_{\theta} \mathbb{E}_{(x,y)\sim D_f}\mathcal{L}(\theta, x, y)
\end{aligned}  
\end{equation}
where $\mathcal{L}$ denotes training loss (\textit{e.g.} reconstruction loss).
Meanwhile, assume we could \textit{define} the influence of a single data on the weight parameter and generation result. Then, unlearning this data would be by simply updating the weight parameter in a direction of removing the data influence. Toward this, we start with defining the data influence on weight parameters and approximates to feasible form as introduced in \citet{koh2017understanding}:
\begin{definition} 
Given upweighting $z$ by some small $\epsilon$ and the new parameters $\hat{\theta}_{\epsilon,z} \myeq \operatorname*{argmin}_{\theta \in \Theta}  \frac{1}{n} \sum_{i=1}^{n} \mathcal{L}(z_i, \theta) + \epsilon \mathcal{L}(z,\theta)$, the influence of upweighting $z$ on the parameter $\hat{\theta}$ is given by 
\begin{align}
I_{up,param}(z) \myeq \frac{d\hat{\theta}_{\epsilon,z}}{d\epsilon}\Bigr|_{\epsilon=0} \myeq -H_{\hat{\theta}}^{-1} \nabla_{\theta} L(z,\hat{\theta}) 
\end{align}
where $H_{\hat{\theta}} = \frac{1}{n}\sum_{i=1}^{n} \nabla_{\theta}^2 L(z_i, \hat{\theta})$ is the Hessian and is positive definite (PD) by assumption. 
\end{definition}
By forming a quadratic approximation to the empirical risk around $\hat{\theta}$, a data influence on the weight parameter is formulated as a single Newtons step (See details in Appendix of \citep{koh2017understanding}), which is consistent with the objective we have mentioned in Equation~\ref{eq:eq1}. Although numerous works have verified that this data influence-based approach works well in shallow, discriminative models~\citep{guo2020certified,golatkar2020eternal,golatkar2020forgetting}, we cannot apply this directly to our generative model due to intractable computation and lack of guarantees on bounds. 
To address this problem, we re-purpose our objective to minimize the \textbf{data influence on generation}. Grounded by recent works~\citep{basu2020second, sun2023regularizing}, we find that we could enjoy this on generative model simply by diminishing the gradient conflict as follows: 
\begin{theorem} \label{thm:gs}
Reducing the influence of samples $z\in D_f$ in training data with regard to test loss is formulated as:
\begin{equation}
\begin{aligned}
\vspace{-0.1in}
I^{'}_{up,loss}(D_f,z') \rightarrow 0, \quad 
\end{aligned}
\end{equation}
which is equivalent to 
\begin{equation}
\begin{aligned}
\quad \nabla_\theta \mathcal{L}(z',\hat{\theta})^T \sum_{z \in D_f} \nabla_\theta \mathcal{L}(z,\hat{\theta}) \rightarrow 0 
\end{aligned}
\end{equation}
where $z'\in D_r$ in our scenario. 
\vspace{-0.1in}
\end{theorem} 
Informally, we could achieve this by alleviating the conflict between two gradients $\nabla_\theta \mathcal{L}(z',\hat{\theta})$ and $\nabla_\theta \mathcal{L}(z,\hat{\theta})$, resulting in diminishing the inner product of two gradients. This reminds us of a classic approach of gradient manipulation techniques for conflicting gradients in multi-task learning scenario ~\citep{yu2020gradient,liu2021conflict,guangyuanrecon}. Specifically, we project a gradient of forget sample $x_f \in D_f$ onto normal plane of a set of retain samples $x_r \in D_r$ to meet $\mathcal{I}_{up,loss}(x_f, x_r)=0$. This orthogonal projection manipulates the original gradient of forget sample $\mathbf{g}_f=\nabla\mathcal{L}_f$ to the weight parameter to which sufficiently unlearns a sample $x_f \in D_f$: $\textbf{g}_f = \textbf{g}_f - \frac{\textbf{g}_f \cdot \textbf{g}_r}{\normx{\textbf{g}_r^2}}\textbf{g}_r$. Then, the unlearned model $\theta^{-}$ is obtained after the following gradient update: $\theta^{-} = \hat{\theta} - \eta\textbf{g}_f$. 
\section{Experiments}
We verify our idea under numerous data removal requests. Note that measuring and evaluating a generative model to unlearn \textit{a single data} is non-trivial. Even comparing pre-trained generative models trained \textit{with} a particular data over \textit{without} simply by looking at the output of training (\textit{e.g.} generated image, weight) is intractable in case of a deep generative model to the best of our knowledge~\citep{van2021memorization}. To make the problem verifiable, in this work, we experiment to unlearn a group of samples sharing similar statistics in the training data - either belonging to a particular class or that has a distinctive semantic feature. In this case, one can evaluate the output of the generation by measuring the number of samples including that class or a semantic feature; a successfully unlearned model would generate nearly zero number of samples having these features. Although we are not able to cover unlearning a single data in this work, note that in essence, our method could successfully approximate the generative model trained without a single data seamlessly, and we look forward to exploring and adjusting a feasible evaluation on this scenario in the near future. 
\begin{table*}[t] \label{tb:qual_table}
  \caption{\small Performance of Class/Feature Unlearning VAE on MNIST138 (\textit{left columns}) and CelebA (\textit{right column}) Each experiments are three times repeated. (*) indicates erroneous evaluation by a pre-trained feature classifier. \textbf{Bold} indicates the best score. }  
  \centering 
  \begin{center}
\begin{small}
\begin{sc}
\resizebox{1.0\linewidth}{!}{
  \begin{tabular}{lcccccccc}
    \toprule
     & \multicolumn{4}{c}{MNIST138(Class: 1)} & \multicolumn{4}{c}{CelebA(Feature: Male)} \\
    \cmidrule(r){2-9} 
    Metric & Privacy & \multicolumn{2}{c}{Utility} & Cost & Privacy & \multicolumn{2}{c}{Utility} & Cost \\
    \cmidrule(r){2-9}
    \text{} & \textit{fratio}($\downarrow$) & \textit{IS}($\uparrow$) & \textit{FID}($\downarrow$) & \textit{Time}(s)($\downarrow$) & \textit{fratio}($\downarrow$) & \textit{IS}($\uparrow$) & \textit{FID}($\downarrow$) & \textit{Time}(s)($\downarrow$) \\
    \midrule
    Before & 0.343(0.027) & 2.053(0.029) & 0.030(0.003) & 218.6 & 0.394(0.119) & 1.812(0.044)  & 29.81(0.341) & $3\times10^4$\\ 
    \midrule
    Grad.Ascnt. & 0.264(0.141) & 2.029(0.018) & 0.127(0.059) & \textbf{1.010} & - (*) & \textbf{1.311}(0.076) & 30.93(1.215) & \textbf{97.31} \\ 
    \citet{moon2023feature}  & 0.344(0.019) & 2.048(0.021) & \textbf{0.031}(0.002) & 166.2 & 1.000(0.000) & 1.000(0.000) & \textbf{15.81}(9.831) & $8\times10^4$ \\ 
    \midrule
    Ours  & \textbf{0.153}(0.057) & \textbf{2.192}(0.076) & 0.092(0.030) & 13.12 & \textbf{0.150}(0.098) & 1.254(0.013) & 34.24(0.698) & 613.2\\ 
    \bottomrule 
    \end{tabular}
}
\end{sc}
\end{small}
\end{center}
\vskip -0.1in
\end{table*}
\subsection{Experimental Setup}
\begin{description}[leftmargin=0pt]
    \item[Scenarios]{
     We unlearn either a whole class or some notable feature from a group of samples. In the experiment, we use a subset of MNIST~\citep{alsaafin2017minimal} with samples of classes 1,3,8 
     and 64x64 CelebA~\citep{liu2015deep} to train and unlearn vanilla VAE~\citep{kingma2013auto}.
    }
    \vspace{-0.1in}
    \item[Evaluation]{ 
    We evaluate our method under the following three criteria: a privacy guarantee, utility guarantee, and cost. Privacy guarantee includes feature ratio (\textit{ fratio}), a ratio of images including the target feature (See details in Appendix~\ref{appdx:experiment}). Utility guarantee includes Frechet Inception Distance (\textit{FID}), a widely used measure for generation quality. Cost includes a total execution time (\textit{Time}) which should be shorter than retrain-from-scratch. A successfully unlearned model would show near-zero on feature ratio, the same IS, FID score as the initial pre-trained model (BEFORE), and the lowest possible execution time. Given the legal impact and the goal of unlearning, note that guaranteeing privacy is prioritized the highest. 
    } 
\end{description}

\begin{figure} [h]\label{fig:gs_mnist_final}
\centerline{\includegraphics[width=0.99\columnwidth]{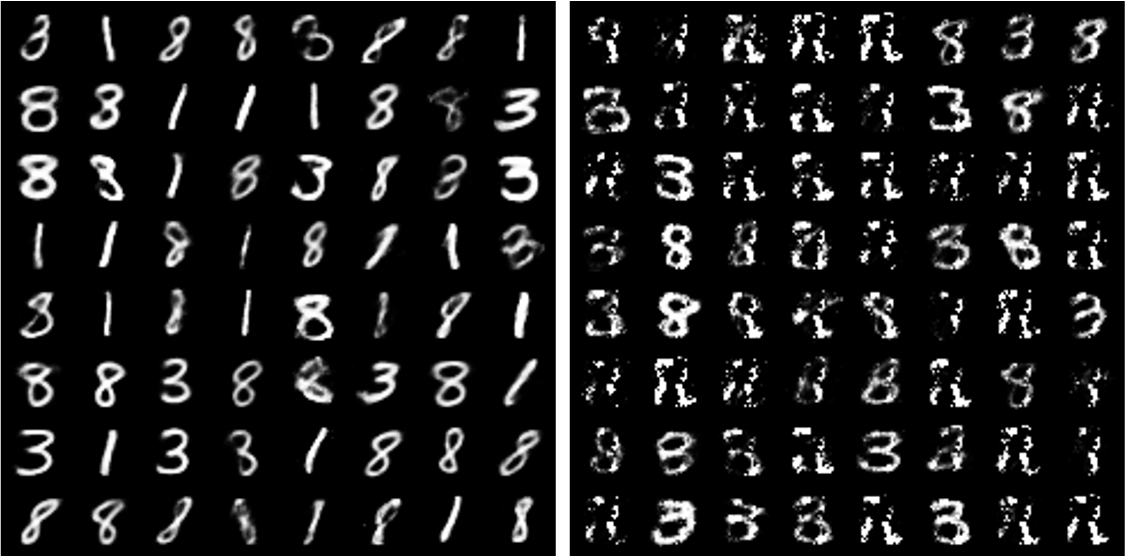}}
\caption{\small Unlearning groups of class 1 samples from VAE pre-trained on MNIST138 (\textit{left}: original, \textit{right}: unlearned) Note that images of class 1 do not appear in generation result.}
\end{figure}


\subsection{Result on Pre-trained Generative Model}
\paragraph{Quantitative Result}{ We run the proposed method on pre-trained VAE to remove unlearning group $D_f$ (\textit{e.g.} class 1 or male, respectively) and evaluate them as follows (Table~\ref{tb:qual_table}) Starting from the pre-trained model (BEFORE) our method unlearns the target $D_f$ with a large decrease on $\textit{fratio}$ by 65\% to 70\% while keeping the time cost of unlearning $\leq$ 5\% of retrain-from-scratch. 
All the while, our method still keeps a decent utility performance. Comparing the baselines, our method shows the best in privacy - the prioritized metric - through all experiments. Note that the feature ratio of gradient ascent in the CelebA experiment (feature ratio-CelebA-Grad.Ascnt) was omitted because the generated samples are turned out to be noisy images and thus the evaluation result of pre-trained classifier cannot be accepted. Also, note that although baselines show better performance in terms of utility and cost, they don't show near-best score on privacy guarantee. 

\paragraph{Qualitative Result}{
We further validate our method by comparing the generated images before and after the proposed unlearning algorithm. As in Figure~\ref{fig:gs_mnist_final}, no class 1 samples are observed after unlearning class 1, meaning that our method successfully meets the request of unlearning class 1, which aligns with the quantitative result where the ratio of samples with class 1 is reduced from 34.3\% to $\leq$ 15\% as in Table~\ref{tb:qual_table}. The output of image generation is fair where 3 and 8 are decently distinguishable through one's eyes, although it is certain that some examples show some minor damaged features, which are in the same line as a decrease in IS and an increase in FID score. Note that the ultimate goal of unlearning is to meet the privacy guarantee while preserving the utility of pre-training, which are remained as our next future work.
}

\section{Conclusion}
In this work, we introduce a novel theoretically sounded unlearning method for the generative method. Inspired by the influence of the sample on the others, we suggest a simple and effective gradient surgery to unlearn a given set of samples on a pre-trained generative model and outperform the existing baselines. Although we don't experiment to unlearn single data due to a lack of ground evaluation on the uniqueness of the particular data, we leave it as future work emphasizing that our method could also be applied to this scenario. Furthermore, it would be interesting to verify our ideas on various privacy-sensitive datasets. Nonetheless, our work implies the possibility of unlearning a pre-trained generative model, laying the groundwork for privacy handling in generative AI. 


\nocite{bishop1992exact}
\nocite{goodfellow2013multi}
\nocite{fu2022knowledge}
\nocite{liu2021federaser}
\nocite{gupta2021adaptive}
\nocite{bourtoule2021machine}
\nocite{zhang2022prompt}

\bibliography{references}
\bibliographystyle{icml2023}

\newpage
\appendix
\onecolumn

\section{Experimental Details} \label{appdx:experiment} 
\subsection{Setup}
\paragraph{Architecture}
 In this experiment, we use vanilla VAE~\citep{kingma2013auto} with encoders of either stack of linear(for MNIST experiment) or convolutional(for CelebA experiment) layers. Although we verify our result on VAE, note that our method can be applied to any variational inference based generative model such as~\citep{kingma2021variational,higgins2017beta}. 

\paragraph{Baseline}
We compare our experimental results with the following two baselines. One is a recently published, first and the only unlearning work on generative model~\citep{moon2023feature} (\textit{FU}) to unlearn by feeding a surrogate model with projected latent vectors. We reproduce FU and follow the hyperparameter details (\textit{e.g.} unlearning epochs 200 for MNIST) as in the original paper. The other is a straight-forward baseline (\textit{Grad.Ascnt.}) which updates the gradient in a direction of maximizing the reconstruction loss on forget, which is equivalent to meeting $\textit{e.g.}$ Objective ~\ref{eq:eq1} without gradient surgery. Note that we keep the same step size when unlearning with these three different methods (including ours) for fair comparison.  

\paragraph{Training details}{
We use Adam optimizer with learning rate 5e-04 for MNIST experiment and 1e-05 for CelebA experiment. We update the parameter only once (1 epoch) for removals, thus named our title 'one-shot unlearning'. All experiments are three times repeated.
}

\subsection{How to Evaluate Feature Ratio}
We first prepare a classification model that classifies the image having a target feature from the remains. In order to obtain a highly accurate classifier, we search for the best classifier which shows over 95\% accuracy. In the experiment, we use AllCNN~\citep{springenberg2014striving} to classify class 1 over the other in MNIST with 1,3,8 (MNIST381), and ResNet18~\citep{he2016deep} to classify male over female on CelebA. After unlearning, we generate 10000 samples from the generator and feed the sample to the pre-trained classifier. Assuming that the classifier classifies the image well, the prediction result would the probability that the generated output contains the features to be unlearned. 

\section{Definitions and Proof for Theoretical Analysis} \label{appdx:gs}

In ~\citet{koh2017understanding} and ~\citet{basu2020second}, an influence of sample  $z$ on weight parameter is defined as the product of its gradient and inverse of hessian. Moreover, an influence of sample $z$ to \textit{test loss} of sample $z'$ defined in as following: 
\begin{definition} (Equation 2 from \citet{koh2017understanding})
Suppose up-weighting a converged parameter $\hat{\theta}$ by small $\epsilon$, which gives us new parameters $\hat{\theta}_{\epsilon,z} \myeq \operatorname*{argmin}_{\theta \in \Theta}  \frac{1}{n} \sum_{i=1}^{n} \mathcal{L}(z_i, \theta) + \epsilon \mathcal{L}(z,\theta)$. The influence of up-weighting $z$ on the loss at an arbitrary point $z'$ against has a closed-form expression: 
\begin{equation}
\begin{aligned}
\mathcal{I}_{up,loss}(z, z') \quad \myeq \quad \frac{d\mathcal{L}(z',\hat{\theta}_{\epsilon,z})}{d\epsilon}\Bigr|_{\epsilon=0} 
\quad \\
=\quad \nabla_\theta \mathcal{L}(z',\hat{\theta})^\top H_{\hat{\theta}}^{-1} \nabla_\theta \mathcal{L}(z,\hat{\theta})
\end{aligned}
\end{equation}
\end{definition}
where $H_{\hat{\theta}} \myeq \frac{1}{n}\sum_{i=1}^{n}\nabla_{\theta}^2\mathcal{L}(z_i, \hat{\theta})$ is the Hessian and is positive definite (PD) by assumption on convex and Lipschitz continuity of loss $\mathcal{L}$. 

\begin{theorem} (Theorem~\ref{thm:gs} from Section~\ref{sec:method})
Reducing the influence of samples $z\in D_f$ in training data with regard to test loss is formulated as:
\begin{equation}
\begin{aligned}
\vspace{-0.1in}
I^{'}_{up,loss}(D_f,z') \rightarrow 0, \quad 
\end{aligned}
\end{equation}
which is equivalent to 
\begin{equation}
\begin{aligned}
\quad \nabla_\theta \mathcal{L}(z',\hat{\theta})^T \sum_{z \in D_f} \nabla_\theta \mathcal{L}(z,\hat{\theta}) \rightarrow 0 
\end{aligned}
\end{equation}
where $z'\in D_r$ in our scenario. 
\end{theorem} 


\begin{proof}
The second-order influence of $D_f$, $\mathcal{I}^{(2)}_{up, param}$, is formulated as sum of first-order influence $\mathcal{I}^{(1)}_{up, param}$ and $\mathcal{I}^{'
}_{up, param}$, which captures the dependency of the terms in $\mathcal{O}(\epsilon^2)$ on the group influence is defined as following: 
\begin{equation}
\mathcal{I}^{'}_{up, param}(D_f,z') = \mathcal{A} H_{\hat{\theta}}^{-1} \sum_{z \in D_f} \nabla_\theta \mathcal{L}(z,\hat{\theta}) 
\end{equation}
where $\mathcal{A} = \frac{p}{1-p}(I-(\nabla^2 L(\theta^*))^{-1}\frac{1}{\mathcal{|U|}}\sum_{z\in \mathcal{U}}\nabla^2 \textit{l}(h_{\theta^*}(z)))$ (from \citet{basu2020second}).

The influence of samples in $D_f$ on the test loss of $z'$ can be formulated as:
\begin{equation}
\mathcal{I}_{up, loss}(D_f,z') = \nabla_\theta \mathcal{L}(z,\hat{\theta})^T \mathcal{I}_{up, param}(D_f) 
\end{equation}
which can be equivalently applied to all orders of $\mathcal{I}$ including $\mathcal{I}^{(1)}, \mathcal{I}^{(2)}, \mathcal{I}^{'}$. 

Then, $\mathcal{I}^{'}_{up, loss}(D_f,z') = 0$ is now reduced to 
\begin{equation}
\nabla_\theta \mathcal{L}(z,\hat{\theta})^T \mathcal{A} H_{\hat{\theta}}^{-1} \sum_{z \in D_f}\nabla_\theta \mathcal{L}(z,\hat{\theta}) = 0
\end{equation}

which satisfies the right-hand side of Theorem~\ref{thm:gs} where $\mathcal{A}$ and $H_{\hat{\theta}}^{-1}$ are negligible.
\end{proof}




\end{document}